\documentclass[letterpaper]{article} 

\ifdefined\aaaianonymous
    \usepackage[submission]{aaai2027}  
\else
    \usepackage{aaai2027}              
\fi
\nocopyright  

\usepackage{times}  
\usepackage{helvet}  
\usepackage{courier}  
\usepackage[hyphens]{url}  
\usepackage{graphicx} 
\usepackage{natbib}  
\usepackage{caption} 
\usepackage{amsmath}
\usepackage{amssymb}
\usepackage{amsthm}
\usepackage{booktabs}

\newtheorem{proposition}{Proposition}
\newtheorem{lemma}{Lemma}
\newtheorem{corollary}{Corollary}
\theoremstyle{remark}
\newtheorem{remark}{Remark}

\newcommand{\E}{\mathbb{E}}

\newcommand{\Prob}{\mathbb{P}}
\newcommand{\Var}{\operatorname{Var}}
\newcommand{\sign}{\operatorname{sign}}
\newcommand{\grad}{\nabla f}
\newcommand{\cf}{\mathrm{cf}}

\title{Same Loss, Same Noise, Opposite Schedules:\\
Noise Structure and Optimizer Normalization Jointly\\
Determine Whether Learning-Rate Cooldown Helps}

\author{
    Subham Singh\textsuperscript{\rm 1},
    Ashutosh Mishra\textsuperscript{\rm 2},
    Subha Raut\textsuperscript{\rm 1}
}
\affiliations{
    \textsuperscript{\rm 1}Department of Mathematics and Statistics, Mississippi State University\\
    \textsuperscript{\rm 2}Independent Researcher\\
    ss4707@msstate.edu, ashutoshm1771@gmail.com, sr2120@msstate.edu
}

\begin{document}

\maketitle

\begin{abstract}
The cooldown phase of a warmup--stable--decay (WSD) learning-rate schedule, now a default in large-model pretraining, lowers the final training loss in some settings and does nothing in others. We give a provable account of which case obtains, and it turns on two properties together: the structure of the gradient noise and whether the optimizer normalizes its update. On a strongly convex objective with multiplicative (gradient-proportional) noise, stochastic gradient descent contracts geometrically at a constant learning rate, so cooldown has nothing to improve. Under the same objective and noise, sign-based and normalized methods, the standard surrogates for adaptive optimizers, settle on a noise floor of order $\eta^2$ and reach the minimizer only as the learning rate is driven to zero; any additive noise then reinstates a floor for every method. The mechanism is elementary: an SGD step shrinks in proportion to the gradient and so anneals itself, whereas a normalized step keeps unit scale and cannot. We solve the signSGD stationary law on the quadratic exactly and obtain the floor constant in closed form, prove a local form of the dissociation under $(L_0,L_1)$-smoothness, extend the floor to normalized SGD in dimension $d>1$ by a scale-invariance argument, and establish robustness to momentum and heavy-tailed noise. Simulation confirms every prediction, and we demonstrate the resulting noise-regime diagnostic on a real classification task with directly measured gradient noise. The mechanism explains \emph{whether} cooldown helps; the \emph{interior} cooldown fraction used at scale lies outside stationary landscape-and-noise geometry.
\end{abstract}

\section{Introduction}

The learning-rate schedule is among the most consequential and least understood choices in neural network training. Warmup--stable--decay (WSD) schedules hold the rate constant for most of training and then cool it down to a small final value \citep{hu2024minicpm}. They have become a default for large language model pretraining, where they match or exceed cosine schedules while letting training resume from any stable-phase checkpoint \citep{haegele2024}, and their qualitative dynamics are strikingly consistent across architectures \citep{belloni2026}. What remains missing is a theoretical account of \emph{when} the cooldown phase actually helps.

Part of the difficulty is that cooldown is a last-iterate effect: it lowers the loss of the final iterate without improving the best loss seen so far. It therefore leaves no trace in the min-over-iterates and randomly-sampled-iterate guarantees that dominate nonconvex optimization theory \citep{compagnoni2026,shulgin2026}, and even sharp last-iterate analyses of SGD \citep{shamir2013} concern the additive-noise regime. Convex last-iterate bounds do reproduce the shape of practical schedules remarkably well \citep{schaipp2025}, but they too treat (sub)gradient methods under additive-type noise. Nothing in either line of work predicts that changing the \emph{optimizer}, with the loss and the noise held fixed, should change whether cooldown helps at all.

This paper shows that it does. We do not try to resolve the WSD cooldown fraction observed at scale. We isolate a narrower question that admits sharp answers: with the loss and the noise held fixed, does the optimizer alone decide whether cooldown helps? It does, and the deciding property is whether the magnitude of the update shrinks as the gradient shrinks.

Our setting is the affine-variance noise model of \citet[Def.~3.2]{compagnoni2026}, which bounds the gradient-noise variance by $\sigma_0^2+\sigma_1^2\|\grad\|^2$. The $\sigma_0^2$ term is the usual bounded additive noise. The $\sigma_1^2$ term scales with the gradient and vanishes at the optimum, as in the interpolation regime of over-parameterized models \citep{schmidt2013,vaswani2019}. Our contributions are as follows.

\begin{itemize}
\item \textbf{SGD self-anneals under multiplicative noise.} With $\sigma_0=0$ the second moment obeys an exactly multiplicative recursion with no additive term (Prop.~\ref{prop:sgd}): the optimal fixed-horizon schedule is the constant rate $\eta^\star=1/(\lambda(1+\sigma_1^2))$, convergence is geometric, and the optimal cooldown fraction is $0$. This half restates classical interpolation results in schedule language, and we claim no novelty for it.
\item \textbf{Normalized methods cannot self-anneal.} Under the same noise, the probability that signSGD takes a correct step is independent of the iterate, for \emph{any} noise law (Lemma~\ref{lem:sign}). On the quadratic the induced dynamics form an explicitly solvable lattice walk: the last-iterate loss is floored at $\Theta(\eta^2)$, with the floor constant in closed form (Prop.~\ref{prop:sign}, Remark~\ref{rem:const}). Cooldown is necessary. Any additive component ($\sigma_0>0$) restores an $\eta$-dependent floor for every method (Prop.~\ref{prop:additive}).
\item \textbf{Robustness.} A local version of the dissociation holds under $(L_0,L_1)$-smoothness (Sec.~\ref{sec:l0l1}); normalized SGD in $d>1$ has a dimension-free granularity floor, with a scale-invariance argument for why it cannot anneal (Sec.~\ref{sec:highdim}); the classification is inherited under momentum (Sec.~\ref{sec:momentum}) and moment-free under heavy tails (Sec.~\ref{sec:heavytail}).
\item \textbf{Verification and scope.} Every quantitative prediction is checked against simulation, including WSD sweeps on a nonconvex deep linear network, and the noise structure and schedule preferences are measured directly on a real classification task (Sec.~\ref{sec:exp}, Figs.~\ref{fig:diss} and~\ref{fig:mlp}). We state plainly what the mechanism does not explain (Sec.~\ref{sec:scope}).
\end{itemize}

\begin{figure*}[t]
\centering
\includegraphics[width=0.90\textwidth]{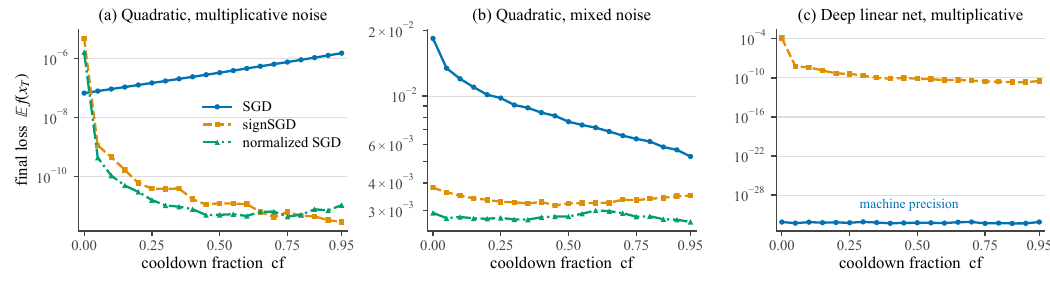}
\caption{Last-iterate loss versus cooldown fraction under WSD schedules, with the peak rate tuned separately for every point (mean over $512$ runs for (a)--(b), $8$ runs for (c)). (a)~Anisotropic quadratic ($d=20$), pure multiplicative noise: cooldown strictly hurts SGD and improves signSGD and normalized SGD by about six orders of magnitude. (b)~Same quadratic with an added additive component ($\sigma_0=0.2$): every method now benefits from cooldown. (c)~Deep linear network with a planted teacher, multiplicative noise: SGD sits at numerical zero ($\approx10^{-32}$) for every cooldown fraction while signSGD requires strong cooldown.}
\label{fig:diss}
\end{figure*}

\section{Setup and Noise Model}\label{sec:setup}

We minimize the separable quadratic $f(x)=\tfrac12\sum_{i=1}^d\lambda_i x_i^2$ with $\lambda_i>0$; Section~\ref{sec:l0l1} treats an $(L_0,L_1)$-smooth extension. Gradients are observed with noise, $\hat g_t=\grad(x_t)+\xi_t$. Following \citet[Def.~3.2]{compagnoni2026}, who bound the gradient-noise variance by $\sigma_0^2+\sigma_1^2\|\grad\|^2$, we work with the coordinatewise instance of that model in which the bound is met exactly:
\begin{equation}\label{eq:noise}
\xi_i=\sigma_0 Z^0_i+\sigma_1\,|\partial_i f(x)|\,Z_i ,
\end{equation}
where the joint law of $(Z^0,Z)$ does not depend on $x$ and, whenever the moments exist, $\E[Z_i]=\E[Z^0_i]=0$ and $\E[Z_i^2]=\E[(Z^0_i)^2]=1$. ``Multiplicative'' means $\sigma_0=0$. In that case the noise is a \emph{scale family}: the standardized noise $Z$ is iterate-independent, and the noise vanishes at the optimum. This structural property, not Gaussianity, is what our results use; Lemma~\ref{lem:sign} and the floors below hold for arbitrary $Z$, including laws without any finite moments (Sec.~\ref{sec:heavytail}).

We study the last-iterate loss $\E[f(x_T)]$ under WSD schedules: $\eta_t=\eta$ for $t<(1-\cf)T$, after which $\eta_t$ decays linearly to $0$, with peak rate $\eta$ and cooldown fraction $\cf\in[0,1]$.

\section{SGD Self-Anneals under Multiplicative Noise}\label{sec:sgd}

\begin{proposition}[SGD, pure multiplicative noise]\label{prop:sgd}
Consider one coordinate with curvature $\lambda$ and $\sigma_0=0$, and write $m_t:=\E[x_t^2]$. Then
\begin{equation}\label{eq:sgdrec}
m_{t+1}=r(\eta_t)\,m_t,\qquad
m_T=m_0\textstyle\prod_{t<T}r(\eta_t),
\end{equation}
with $r(\eta)=(1-\eta\lambda)^2+\eta^2\sigma_1^2\lambda^2$, and:
(i) $r$ is minimized at $\eta^\star=1/(\lambda(1+\sigma_1^2))$, where $r(\eta^\star)=\sigma_1^2/(1+\sigma_1^2)<1$;
(ii) over schedules $\{\eta_t\in[0,\eta_{\max}]\}_{t<T}$ with $T$ fixed, $m_T$ is minimized by the constant schedule $\eta_t\equiv\min(\eta^\star,\eta_{\max})$, so the optimal cooldown fraction within the WSD family is $0$;
(iii) at $\eta^\star$, $\E[f(x_T)]=f(x_0)\big(\sigma_1^2/(1+\sigma_1^2)\big)^T\to0$ geometrically, with no noise floor. In $d$ dimensions every coordinate contracts geometrically at any constant rate with $\max_i r_i(\eta)<1$, so the loss again decays geometrically to zero.
\end{proposition}

\begin{proof}
From $\hat g_t=\lambda x_t+\xi_t$ we get $x_{t+1}=(1-\eta_t\lambda)x_t-\eta_t\xi_t$; conditioning on $x_t$ and using $\E[\xi_t\mid x_t]=0$, $\E[\xi_t^2\mid x_t]=\sigma_1^2\lambda^2x_t^2$ gives \eqref{eq:sgdrec}. For (i), $r'(\eta)=-2\lambda+2\eta\lambda^2(1+\sigma_1^2)$ vanishes at $\eta^\star$, and substituting $1-\eta^\star\lambda=\sigma_1^2/(1+\sigma_1^2)$ gives $r(\eta^\star)$. For (ii), $\log m_T=\log m_0+\sum_t\log r(\eta_t)$ is minimized term by term, and $r$ is decreasing on $[0,\eta^\star]$. Part (iii) applies \eqref{eq:sgdrec} coordinatewise and sums $\tfrac12\lambda_i\E[x_{i,T}^2]$.
\end{proof}

The noise contribution $\eta_t^2\sigma_1^2\lambda^2m_t$ is proportional to $m_t$, so it forms part of the geometric contraction rather than an additive floor. Near the optimum the SGD step $\eta_t\hat g_t\propto\eta_t x_t$ shrinks on its own, and the schedule has nothing further to contribute.

\begin{remark}[Constraint set for the schedule optimization]\label{rem:constraint}
Part (ii) optimizes the fixed-horizon last-iterate loss with no budget constraint on $\sum_t\eta_t$; on exactly this set the objective is minimized termwise, hence by a constant schedule. A budget on $\sum_t\eta_t$ would change the problem. Warmup plays no role here; it is needed only for stability under $(L_0,L_1)$ curvature (Sec.~\ref{sec:l0l1}). Proposition~\ref{prop:sgd} itself restates the linear convergence of constant-step SGD under interpolation \citep{schmidt2013,vaswani2019} in schedule language; it is the baseline for the contrast that follows, not a contribution.
\end{remark}

\section{Normalized Methods Cannot Self-Anneal}\label{sec:sign}

\begin{lemma}[Iterate-independent sign reliability]\label{lem:sign}
Under scale-family multiplicative noise on one coordinate,
\begin{equation}
\Prob\big(\sign(\hat g)=\sign(x)\big)=
\begin{cases}
\Prob(Z>-1/\sigma_1), & x>0,\\[2pt]
\Prob(Z<\phantom{-}1/\sigma_1), & x<0,
\end{cases}
\end{equation}
independent of $|x|$ for every law of $Z$. If $Z$ is symmetric the two values coincide; call them $p$. Gaussian $Z$ gives $p=\Phi(1/\sigma_1)$, and in general $p>\tfrac12$ if and only if $\Prob(Z\le-1/\sigma_1)<\tfrac12$, which holds in particular whenever $Z$ has median zero and positive mass on $(-1/\sigma_1,0]$.
\end{lemma}

\begin{proof}
For $x>0$, $\hat g=\lambda x(1+\sigma_1 Z)$, so $\sign(\hat g)=\sign(1+\sigma_1Z)$; the case $x<0$ is analogous. (Ties $1+\sigma_1Z=0$ have probability zero for continuous $Z$.)
\end{proof}

The probability of a correct, restoring step does not improve as the iterate approaches the optimum. Taking the sign discards precisely the gradient-magnitude information that produces SGD's self-annealing.

\begin{proposition}[signSGD floor]\label{prop:sign}
Run signSGD, $x_{t+1}=x_t-\eta\,\sign(\hat g_t)$, with constant $\eta$ on one coordinate with curvature $\lambda>0$, pure multiplicative noise with symmetric $Z$, and $x_0=(k+\varphi)\eta\notin\eta\mathbb Z$ with offset $\varphi\in(0,1)$ (any continuous initialization satisfies this almost surely). Then:

(i) \emph{(Granularity.)} Every step changes $x$ by exactly $\pm\eta$, so $\max(|x_t|,|x_{t+1}|)\ge\eta/2$ for all $t$, and
\[
\E[f(x_T)]+\E[f(x_{T+1})]\;\ge\;\tfrac18\lambda\eta^2
\quad\text{for every }T,
\]
hence $\limsup_T\E[f(x_T)]\ge\tfrac1{16}\lambda\eta^2$: the loss cannot converge to zero at constant $\eta$.

(ii) \emph{(Exact floor.)} $|x_t|$ is a reflected random walk on the lattice $\{(j+\varphi)\eta\}_{j\ge0}\cup\{(j+1-\varphi)\eta\}_{j\ge0}$ with iterate-independent inward probability $p$ (Lemma~\ref{lem:sign}). For $p\in(\tfrac12,1]$ its invariant law is explicit, and under it, with $q:=1-p$,
\begin{equation}\label{eq:cexact}
\begin{gathered}
\E[x^2]=c(\varphi,p)\,\eta^2,\\
c(\varphi,p)=\frac{\varphi^2+(1-\varphi)^2}{2}+\frac{q}{2p-1}+\frac{q}{(2p-1)^2},
\end{gathered}
\end{equation}
so the stationary loss is $\tfrac12\lambda\,c(\varphi,p)\,\eta^2=\Theta(\eta^2)$. Driving the loss to zero requires $\eta_t\to0$.
\end{proposition}

\begin{proof}[Proof sketch]
Since $x_0\notin\eta\mathbb Z$, the iterate never hits $0$, $\hat g_t\neq0$ almost surely, and each step is exactly $\pm\eta$; if $|x_t|<\eta/2$ then $|x_{t+1}|=\eta-|x_t|>\eta/2$, which gives $x_t^2+x_{t+1}^2\ge\eta^2/4$ pointwise and hence (i). For (ii), Lemma~\ref{lem:sign} makes the inward/outward indicators i.i.d.\ Bernoulli($p$), so $|x_t|$ is a birth--death chain on two half-lattices joined at the bottom (the walk crosses $0$ between $\varphi\eta$ and $(1-\varphi)\eta$). A chain on a tree is reversible; detailed balance gives geometric level weights proportional to $(q/p)^{j}$ on both half-lattices with equal weight at the two bottom states, and summing $x^2$ against these weights yields \eqref{eq:cexact}; the full computation is in the technical appendix.
\end{proof}

\begin{corollary}[Coordinatewise signSGD, $d\ge1$]\label{cor:signd}
On the separable quadratic with $x_{0,i}\notin\eta\mathbb Z$ for every $i$, coordinatewise signSGD moves each coordinate by exactly $\pm\eta$ per step, so Proposition~\ref{prop:sign}(i) applies coordinatewise and $\E[f(x_T)]+\E[f(x_{T+1})]\ge\tfrac{\eta^2}{8}\sum_{i=1}^d\lambda_i$ for every $T$. The floor grows with the trace of the curvature and holds for any noise law.
\end{corollary}

\begin{proof}
$\hat g_i=\lambda_ix_i(1+\sigma_1Z_i)\neq0$ almost surely, so each coordinate is a $\pm\eta$ walk on its own lattice; apply the pair bound of Proposition~\ref{prop:sign}(i) per coordinate and sum.
\end{proof}

\begin{remark}[What is new here]\label{rem:folklore}
That constant-step signSGD can fail to converge is not by itself new: \citet{karimireddy2019} give convex counterexamples, and \citet{bernstein2018} use decaying steps for exactly this reason. Proposition~\ref{prop:sign} sharpens the phenomenon in two ways that matter for schedules: the stationary law is solved exactly, with the floor constant in closed form, and Lemma~\ref{lem:sign} shows the obstruction is iterate-independent for every noise law. Combined with Proposition~\ref{prop:sgd}, this yields the point of the paper: under the very noise for which signSGD is floored, SGD needs no schedule at all.
\end{remark}

\begin{remark}[The floor constant]\label{rem:const}
Averaging \eqref{eq:cexact} over a uniform offset, as under continuous initialization, gives
\[
\bar c(p)=\tfrac13+\frac{q}{2p-1}+\frac{q}{(2p-1)^2},
\]
which diverges like $q/(2p-1)^2$ as $p\downarrow\tfrac12$ and decreases to $\tfrac13$ as $p\to1$; also $c(\varphi,p)\ge\tfrac14$ always. The constant is testable: simulations match it to within Monte Carlo error across the whole range (Table~\ref{tab:cp}). A drift--diffusion (small-step) approximation, by contrast, predicts $8p^2q^2/(2p-1)^2$, an $88\times$ underestimate at $p=\Phi(2)\approx0.977$: the step size equals the stationary width $\Theta(\eta)$, so there is no scale separation and the chain must be solved exactly, as \eqref{eq:cexact} does (technical appendix).
\end{remark}

\paragraph{Mechanism.}
The SGD step $\eta\hat g\propto\eta x$ shrinks with the iterate, whereas the signSGD step $\eta\,\sign(\hat g)$ keeps a fixed magnitude however close the iterate is to the optimum. Any annealing the update does not provide on its own must instead come from the schedule.

\section{The Dissociation and the Additive Boundary}\label{sec:diss}

\begin{table}[t]
\centering
\footnotesize
\caption{Optimal schedule by noise structure and optimizer.}
\label{tab:diss}
\setlength{\tabcolsep}{3pt}
\begin{tabular}{lll}
\toprule
Noise structure & SGD & signSGD / normalized \\
\midrule
Multiplicative ($\sigma_0=0$) & constant rate & cooldown \\
Additive / mixed ($\sigma_0>0$) & cooldown & cooldown \\
\bottomrule
\end{tabular}
\end{table}

\begin{proposition}[Additive regime]\label{prop:additive}
Let $\sigma_0>0$ and $0<\eta<2/(\lambda(1+\sigma_1^2))$, so that $a(\eta):=(1-\eta\lambda)^2+\eta^2\sigma_1^2\lambda^2<1$. The SGD second moment obeys the affine recursion $m_{t+1}=a(\eta)m_t+\eta^2\sigma_0^2$, whose fixed point
\begin{equation}
m_\infty=\frac{\eta^2\sigma_0^2}{1-a(\eta)}
=\frac{\eta\,\sigma_0^2}{2\lambda}\big(1+O(\eta)\big)
\end{equation}
is strictly positive and strictly increasing in $\eta$ on the admissible range. Reducing the terminal rate therefore strictly reduces the last-iterate floor: cooldown helps \emph{every} optimizer.
\end{proposition}

\begin{proof}
Iterating the affine map gives $m_t=a^tm_0+\eta^2\sigma_0^2\sum_{k<t}a^k\to\eta^2\sigma_0^2/(1-a)$, and $1-a(\eta)=\eta\lambda(2-\eta\lambda(1+\sigma_1^2))$. Differentiating $m_\infty$ in $\eta$ gives $2\sigma_0^2\lambda/[\lambda(2-\eta\lambda(1+\sigma_1^2))]^2>0$.
\end{proof}

The dissociation is thus specific to the multiplicative regime: an additive component reintroduces a floor that only decay removes, and the two optimizer families agree once more (Table~\ref{tab:diss}). The same objective under the same noise can call for opposite schedules, and the property that decides between them is whether the update self-anneals. One caveat on Table~\ref{tab:diss}: for the anisotropic quadratic under a single shared rate, the constant-schedule optimality for SGD is proven per coordinate and confirmed numerically (Table~\ref{tab:exp}); the termwise argument of Proposition~\ref{prop:sgd}(ii) is scalar.

\begin{remark}[Intermediate noise exponents]\label{rem:alpha}
Real gradient noise need not sit at either endpoint of the model \eqref{eq:noise}; Section~\ref{sec:exp} measures $\Var(\xi\mid x)\propto\|\grad\|^{2\alpha}$ with $\alpha\approx0.56$ on a real task. The dichotomy degrades gracefully. For one coordinate with $\Var(\xi\mid x)=\sigma^2|f'(x)|^{2\alpha}$ and $\alpha\in(0,1)$, Jensen's inequality bounds the noise term by $\eta^2\sigma^2\lambda^{2\alpha}m_t^\alpha$, so $m_{t+1}\le(1-\eta\lambda)^2m_t+\eta^2\sigma^2\lambda^{2\alpha}m_t^\alpha$; the right-hand map is increasing and concave with a unique positive fixed point, whence
\[
\limsup_T\,\E[x_T^2]\;\le\;C(\lambda,\sigma,\alpha)\,\eta^{1/(1-\alpha)} .
\]
This interpolates the two regimes: $\alpha=0$ recovers the additive floor $\Theta(\eta)$ of Proposition~\ref{prop:additive}, and $\alpha\to1$ recovers the vanishing floor of Proposition~\ref{prop:sgd}. Since the signSGD floor is $\Theta(\eta^2)$ for \emph{every} noise law (Prop.~\ref{prop:sign}(i)), the dissociation persists in floor-exponent form for all $\alpha>\tfrac12$, where SGD's residual floor $\eta^{1/(1-\alpha)}=o(\eta^2)$; the measured $\alpha\approx0.56$ gives order $\eta^{2.3}$.
\end{remark}

One limitation is visible already here. In the pure quadratic, the additive-regime optimum within the WSD family is full decay, not an interior cooldown fraction, and no model in this paper produces an interior optimum (Sec.~\ref{sec:scope}).

\section{A Local Theorem under $(L_0,L_1)$-Smoothness}\label{sec:l0l1}

Let $f(x)=\tfrac{\lambda}{2}x^2+\tfrac{c}{4}x^4$ with $c\ge0$, so $f'(x)=\lambda x+cx^3$ and
$|f''(x)|=\lambda+3cx^2\le 4\lambda+3\sqrt{c/\lambda}\,|f'(x)|$:
the objective is $(L_0,L_1)$-smooth in the sense of \citet{zhang2020}, with $L_0=4\lambda$ and $L_1=3\sqrt{c/\lambda}$, but not $L$-smooth for $c>0$. Write $s(x):=f'(x)/x=\lambda+cx^2\ge\lambda$.

\begin{proposition}[SGD: local geometric convergence, no floor]\label{prop:l0l1sgd}
Assume scale-family multiplicative noise with $|Z|\le M$ almost surely and $\sigma_1M\le1$. Fix a basin $B_\rho=\{|x|\le\rho\}$ and let $s_\rho=\lambda+c\rho^2$. If
\[
\eta\;\le\;\min\!\Big\{\frac{2}{s_\rho(1+\sigma_1 M)},\ \frac{1}{s_\rho(1+\sigma_1^2)}\Big\},
\]
then $B_\rho$ is forward-invariant and $\E[x_{t+1}^2\mid x_t]\le r(\lambda,\eta)\,x_t^2$ with $r(\lambda,\eta)=(1-\eta\lambda)^2+\eta^2\sigma_1^2\lambda^2<1$. Hence $\E[x_t^2]\le r(\lambda,\eta)^t x_0^2\to0$ geometrically, with no floor; the rate is governed by the curvature at the minimum, and its infimum over admissible $\eta$ approaches $\sigma_1^2/(1+\sigma_1^2)$ as $\rho\to0$, recovering Proposition~\ref{prop:sgd}.
\end{proposition}

\begin{proof}
The update is $x_{t+1}=x_t\big(1-\eta s(x_t)(1+\sigma_1 Z_t)\big)$. For $|x_t|\le\rho$ and $|Z_t|\le M$, the factor $\eta s(x_t)(1+\sigma_1Z_t)$ lies in $[\,\eta s(x_t)(1-\sigma_1M),\ \eta s_\rho(1+\sigma_1M)\,]\subseteq[0,2]$, using $\sigma_1M\le1$ for the lower end and the step-size condition for the upper end. Hence $|1-\eta s(x_t)(1+\sigma_1Z_t)|\le1$ and $|x_{t+1}|\le|x_t|\le\rho$: the basin is forward-invariant. Next,
\[
\E[x_{t+1}^2\mid x_t]=x_t^2\,r(s(x_t),\eta),
\]
where $r(s,\eta)=1-2\eta s+\eta^2s^2(1+\sigma_1^2)$. For $\eta\le1/(s_\rho(1+\sigma_1^2))$ we have $\partial_s r=-2\eta+2\eta^2s(1+\sigma_1^2)\le0$ on $[\lambda,s_\rho]$, so $r(s(x_t),\eta)\le r(\lambda,\eta)$, and iterating gives the claim. As $\rho\to0$, $s_\rho\to\lambda$ and the admissible range extends to $\eta^\star=1/(\lambda(1+\sigma_1^2))$, where $r(\lambda,\eta^\star)=\sigma_1^2/(1+\sigma_1^2)$.
\end{proof}

The boundedness assumption with $\sigma_1M\le1$ keeps the noisy gradient sign-consistent inside the basin, enabling almost-sure forward invariance; the contraction inequality itself is algebraic. In-basin contraction is \emph{faster} than $r(\lambda,\eta)$ while $|x|$ is large, since $r$ decreases in $s$.

\begin{proposition}[signSGD: floor persists]\label{prop:l0l1sign}
For the quartic family above (any $c\ge0$), signSGD with constant $\eta$, symmetric $Z$, and $x_0\notin\eta\mathbb Z$ satisfies
$\E[f(x_T)]+\E[f(x_{T+1})]\ge\tfrac18\lambda\eta^2$ for every $T$, hence $\limsup_T\E[f(x_T)]\ge\tfrac1{16}\lambda\eta^2$.
\end{proposition}

\begin{proof}
Since $f'(x)=s(x)x$ with $s>0$, $\sign(\hat g)=\sign(x)\sign(1+\sigma_1Z)$ and every step is exactly $\pm\eta$, so the granularity argument of Proposition~\ref{prop:sign}(i) applies verbatim; then $f(x)\ge\tfrac12\lambda x^2$ pointwise.
\end{proof}

The theorem is genuinely local: outside the basin the $(L_0,L_1)$ curvature makes a fixed large step unstable, and SGD with too large an $\eta$ diverges (Sec.~\ref{sec:exp}). This is precisely the role of warmup; the result is a statement about the post-warmup phase.

\section{Normalized SGD in Higher Dimension}\label{sec:highdim}

Let $f(x)=\tfrac12x^\top\Lambda x$ with $\Lambda=\mathrm{diag}(\lambda_i)\succ0$ and scale-family multiplicative noise \eqref{eq:noise} with $\sigma_0=0$. Consider normalized SGD, $x_{t+1}=x_t-\eta\,\hat g_t/\|\hat g_t\|$ (with $x_{t+1}=x_t$ on the probability-zero event $\hat g_t=0$).

\begin{lemma}[Scale-invariant update geometry]\label{lem:scaleinv}
For $\alpha>0$, replacing $x$ by $\alpha x$ scales $\grad$ and $\xi$ jointly by $\alpha$, so the joint law of the unit vectors $\big(\hat g/\|\hat g\|,\ \grad/\|\grad\|,\ x/\|x\|\big)$ depends on $x$ only through its direction $\hat x=x/\|x\|$. Every angular statistic of the update, in particular $\bar c(\hat x):=\E[\cos\angle(\hat g,x)]$ and $\bar s(\hat x):=\E[\sin^2\angle(\hat g,x)]$, is independent of the radius along a fixed ray.
\end{lemma}

\begin{proposition}[Normalized-SGD floor, $d>1$]\label{prop:normsgd}
Let $\lambda_{\min}=\min_i\lambda_i$ and $R_t=\|x_t\|$.

(i) \emph{(Granularity.)} The step norm is exactly $\eta$, so $\max(R_t,R_{t+1})\ge\eta/2$ for all $t$, and
$\E[f(x_T)]+\E[f(x_{T+1})]\ge\tfrac18\lambda_{\min}\eta^2$ for every $T$; the loss cannot converge to zero at constant $\eta$.

(ii) \emph{(Equilibrium scale.)} With $\theta_t=\angle(\hat g_t,x_t)$, for $R_t\gg\eta$,
\[
\E[R_{t+1}-R_t\mid x_t]=-\eta\,\bar c(\hat x_t)+\frac{\eta^2\,\bar s(\hat x_t)}{2R_t}+O\!\big(\eta^3/R_t^2\big),
\]
with $\bar c,\bar s$ radius-independent by Lemma~\ref{lem:scaleinv}. If $c_-:=\inf_{\hat x}\bar c(\hat x)>0$, the drift is inward whenever $R_t\gg\eta$, so the radius is driven down to, and equilibrates at, scale $\Theta(\eta)$; combined with (i), $\E[f(x_T)]=\Theta(\eta^2)$ and cooldown is required. For isotropic $\Lambda=\lambda I$, $\grad\parallel x$ and $\theta_t$ is also the angle to the gradient.
\end{proposition}

\begin{proof}
(i) $R_{t+1}\ge\|\eta\hat g_t/\|\hat g_t\|\|-R_t=\eta-R_t$, so $R_t<\eta/2$ forces $R_{t+1}>\eta/2$; then $R_t^2+R_{t+1}^2\ge\eta^2/4$ and $f(x)\ge\tfrac12\lambda_{\min}\|x\|^2$. (ii) The law of cosines gives $R_{t+1}=\big(R_t^2-2R_t\eta\cos\theta_t+\eta^2\big)^{1/2}$; expanding in $\eta/R_t$ and taking conditional expectations yields the display, and Lemma~\ref{lem:scaleinv} removes the radius dependence of the coefficients.
\end{proof}

This is the $d>1$ form of the mechanism: multiplicative noise makes the \emph{angular} error of the update scale-free, so a fixed-norm step has an inward efficiency that does not improve as the iterate approaches the optimum, and the radius settles at a multiple of the step size.

\begin{remark}[Positivity and anisotropy]\label{rem:cos}
Unbiased noise gives $\E[\langle\hat g,\grad\rangle\mid x]=\|\grad\|^2>0$, but this does not by itself make $\bar c$ positive: normalization is nonlinear, and a $Z$ with sufficiently negative median makes the normalized drift point outward, exactly as $p<\tfrac12$ does in Lemma~\ref{lem:sign} (in $d=1$, $\bar c=2p-1$). For symmetric $Z$ we observe $\bar c>0$ throughout. Anisotropy matters quantitatively: concentrating gradient energy in a few coordinates lowers the radial alignment by an $O(1)$ factor (measured: $0.49$ for a spectrum spanning three decades versus $0.90$ for an isotropic one at $d=20$, $\sigma_1=0.5$; Sec.~\ref{sec:exp}), enlarging the equilibrium radius by the same factor without changing the $\Theta(\eta)$ scaling.
\end{remark}

\section{Momentum}\label{sec:momentum}

Momentum does not change the classification, which turns on whether the \emph{effective} step magnitude shrinks with $\|\grad\|$. Heavy-ball SGD in its averaged form, $m_t=\beta m_{t-1}+(1-\beta)\hat g_t$ with step $-\eta m_t$, has $m_t\to0$ as $\grad\to0$ under multiplicative noise: it still self-anneals. Sign- and RMS-normalized momentum methods (signSGD with momentum, Adam) retain $\Theta(1)$ update magnitude near the optimum under stationary scale-family noise, up to the $\epsilon$-damping and the non-stationary crossover discussed in Section~\ref{sec:exp}, and the floor persists. Cooldown is thus needed exactly when the effective step does not shrink with the gradient.

\section{Heavy-Tailed Noise}\label{sec:heavytail}

The two sides of the dissociation make sharply different moment demands, and the comparison favors the normalized side. The granularity bounds (Props.~\ref{prop:sign}(i), \ref{prop:l0l1sign}, \ref{prop:normsgd}(i)) and the sign reliability of Lemma~\ref{lem:sign} use no moments of $Z$ at all, and $\cos\theta\in[-1,1]$ is bounded, so the signSGD and normalized-SGD floors are $\Theta(\eta^2)$ even under infinite-variance noise; the closed form \eqref{eq:cexact} still applies, with $p$ read off from the noise law. By contrast, the second-moment recursion of Proposition~\ref{prop:sgd} requires $\E[Z^2]<\infty$, e.g.\ Student-$t_\nu$ with $\nu>2$. For $\nu\le2$ the $L^2$ analysis is unavailable, although the iterate can still converge almost surely ($\log|x_t|$ is a random walk whose increments have finite, typically negative mean; verified at $\nu=1.5$ in Sec.~\ref{sec:exp}). The part of the dissociation that is new here is thus robust precisely in the heavy-tailed regime where the classical self-annealing analysis degrades, consistent with the heavy-tail robustness of sign-based methods emphasized by \citet{compagnoni2026}.

\section{Numerical Verification}\label{sec:exp}

All quantitative claims in this paper were checked against discrete-time simulation. Protocol: $\lambda=1$, $\sigma_1=0.5$, Gaussian $Z$, unless stated; stationary quantities are time averages over at least $10^5$ post-burn-in steps across at least $100$ independent chains; complete configurations, seeds, and code are in the supplementary material.

\paragraph{Self-annealing (Prop.~\ref{prop:sgd}).}
The one-step second-moment ratio matches $r(\eta)$ to within $6\times10^{-5}$ across $\eta\in[0.7,0.9]$ ($2\times10^7$ samples per point), with minimum at $\eta^\star=0.8$, $r(\eta^\star)=0.2$; at $\eta^\star$ trajectories decay to numerical underflow with no floor.

\paragraph{signSGD floor (Prop.~\ref{prop:sign}).}
With continuous initialization, the stationary loss scales as $\eta^{2.00\pm0.01}$ over $\eta\in[0.02,0.16]$ for $\sigma_1\in\{0.5,1,2\}$, with measured $c=\E[x^2]/\eta^2=0.38,\,0.90,\,3.24$ against the closed-form values $\bar c(p)=0.382,\,0.906,\,3.243$ from \eqref{eq:cexact}. An $x_0\in\eta\mathbb Z$ would place the absorbing point $0$ on the lattice and give a spuriously low floor, hence the continuous initialization.

\paragraph{Sign reliability is flat (Lemma~\ref{lem:sign}).}
$\Prob(\text{correct sign})$ at $x=1,\,0.1,\,0.01$: $0.977$ at all three scales for Gaussian $Z$ ($\Phi(2)=0.9772$); $0.908$ for Student-$t_2$ ($1-F_{t_2}(-2)=0.9082$); $1.000$ for a standardized lognormal, which is bounded below by $-0.76>-1/\sigma_1$, so $p=1$ exactly. The value of $p$ depends on the noise law; its independence of $x$ does not.

\paragraph{$(L_0,L_1)$ local theorem (Prop.~\ref{prop:l0l1sgd}).}
In-basin ($x_0=0.3$, $\eta=0.3$, $c=1$), the conditional contraction factor rises from $r(s(x_0),\eta)=0.480$ at initialization toward $r(\lambda,\eta)=0.5125$ within roughly $40$ steps, confirming both the transient speedup ($r$ decreasing in $s$) and the asymptotic rate; out of basin ($\eta=0.8$, $x_0=1.5$) the iterate diverges, confirming the locality caveat and the role of warmup.

\paragraph{Additive floor (Prop.~\ref{prop:additive}).}
At $\sigma_0=0.2$, $\eta=0.05$: simulated stationary second moment $1.0323\times10^{-3}$; predicted fixed point $1.0323\times10^{-3}$.

\paragraph{Normalized SGD, $d=20$ (Prop.~\ref{prop:normsgd}).}
With eigenvalues log-spaced from $1$ to $10^{-3}$ and $x\propto\mathbf 1$, the alignment $\E[\cos\angle(\hat g,\grad)]$ equals $0.9131$ and the radial alignment $\E[\cos\angle(\hat g,x)]$ equals $0.4932$, each identical to four digits at radii $1$, $0.1$, and $0.01$ ($2\times10^6$ draws per radius), as Lemma~\ref{lem:scaleinv} requires; the isotropic radial alignment at the same $d,\sigma_1$ is $0.8976$. The stationary loss scales as $\eta^{2.00}$.

\paragraph{Momentum (Sec.~\ref{sec:momentum}).}
At $\eta=0.02$, $\beta=0.9$ under multiplicative noise: heavy-ball SGD reaches mean $\log_{10}x^2\approx-587$ after $2.5\times10^4$ steps and is still descending (no floor), while signSGD-with-momentum stalls at $\E[f]=4.5\times10^{-4}=\Theta(\eta^2)$.

\paragraph{AdamW.}
The floor is not an artifact of the sign surrogate. AdamW ($\beta_1=0.9$, $\beta_2=0.999$, $\epsilon=10^{-12}$) on the quadratic under multiplicative noise has a last-iterate floor scaling as $\eta^{2.00}$ across $\eta\in\{0.005,0.01,0.02,0.04\}$, essentially unchanged by decoupled weight decay ($\eta^{2.01}$ at weight decay $0.01$). The damping constant adds a predictable refinement: once gradients fall below $\epsilon$, the update $\hat m/(\sqrt{\hat v}+\epsilon)\approx\hat m/\epsilon$ becomes gradient-proportional and self-annealing, and the measured exponent steepens to $2.18$ at $\epsilon=10^{-4}$. AdamW therefore needs cooldown in the multiplicative regime, with $\epsilon$ setting the gradient scale below which it crosses over to SGD-like behavior.

\paragraph{Heavy tails.}
Under Cauchy $Z$ (no finite mean), the signSGD floor still scales as $\eta^{2.00}$, and the measured constant matches $\bar c(p)=0.840$ at $p=\tfrac12+\arctan(1/\sigma_1)/\pi\approx0.852$: the closed form is moment-free, as claimed. For SGD under Student-$t_{1.5}$ multiplicative noise at $\eta=0.8$, the median of $|x_{300}|$ is $\approx10^{-135}$ (almost-sure convergence), while the sample mean of $x^2$ is $\approx10^{-193}$, far above the median $\approx10^{-270}$: rare excursions dominate, as expected when the population second moment is infinite.

\begin{table}[t]
\centering
\small
\caption{signSGD floor constant: closed form $\bar c(p)$ of Remark~\ref{rem:const} versus simulation (time-averaged, continuous initialization).}
\label{tab:cp}
\setlength{\tabcolsep}{4.5pt}
\begin{tabular}{lccccc}
\toprule
$p$ & $0.60$ & $0.70$ & $0.80$ & $0.90$ & $0.95$ \\
$\sigma_1$ & $3.95$ & $1.91$ & $1.19$ & $0.78$ & $0.61$ \\
\midrule
$\bar c(p)$ closed form & $12.33$ & $2.958$ & $1.222$ & $0.615$ & $0.451$ \\
$\bar c(p)$ simulation & $12.33$ & $2.955$ & $1.219$ & $0.612$ & $0.448$ \\
\bottomrule
\end{tabular}
\end{table}

\paragraph{Exact floor constant (Remark~\ref{rem:const}).}
Table~\ref{tab:cp} compares the closed form against simulation across the practical range of $p$; agreement is within $0.7\%$, the Monte Carlo error of the time averages.

\begin{figure}[t]
\centering
\includegraphics[width=\columnwidth]{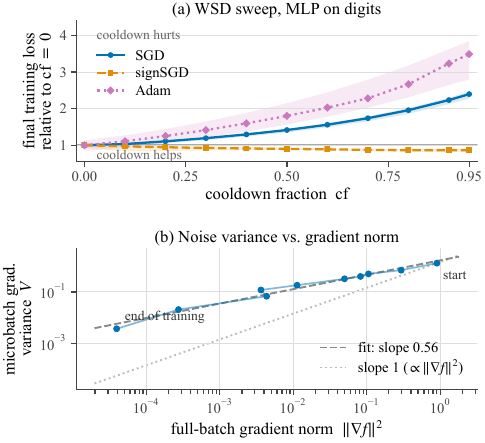}
\caption{Real-data check: two-layer ReLU network on the digits dataset. (a)~Final training loss under WSD schedules, relative to no cooldown (peak rate tuned per point; bands are min--max over five seeds). Cooldown hurts SGD and Adam here and helps only signSGD. (b)~Per-microbatch gradient variance against the full-batch gradient norm along a reference SGD trajectory: the noise collapses together with the gradient (log--log slope $0.56$ against $\|\grad\|^2$; a pure multiplicative model would give slope $1$), so the vanishing-noise regime of Sections~\ref{sec:sgd}--\ref{sec:highdim} is the operative one.}
\label{fig:mlp}
\end{figure}

\paragraph{Dissociation across models (Table~\ref{tab:exp}, Fig.~\ref{fig:diss}).}
We ran full WSD sweeps (cooldown fraction $\cf\in\{0,0.05,\dots,0.95\}$, peak rate tuned per point) on the anisotropic quadratic, the $(L_0,L_1)$ quartic, and a nonconvex deep linear network $y=W_2W_1x$ with a planted orthogonal teacher and multiplicative gradient noise. Under multiplicative noise, cooldown makes SGD strictly worse on the quadratic (from $6.8\times10^{-8}$ at $\cf=0$ to $1.5\times10^{-6}$ at $\cf=0.95$) and is irrelevant on the quartic and the deep linear network, where SGD reaches numerical zero at every $\cf$; signSGD and normalized SGD improve by five to seven orders of magnitude with strong cooldown, with a broad plateau over $\cf\in[0.4,0.95]$. Within such a plateau the argmin is not sharply determined, so Table~\ref{tab:exp} should be read for the contrast between $\cf^\star=0$ and a plateau near full decay rather than for the exact argmin. With $\sigma_0>0$ every method benefits from cooldown. These controlled testbeds confirm the mechanism across convex, $(L_0,L_1)$, and nonconvex landscapes, dimensions, anisotropy, and noise laws.

\paragraph{A real-data check (Fig.~\ref{fig:mlp}).}
We also trained a two-layer ReLU network ($64$--$128$--$10$) on the scikit-learn digits dataset ($1{,}500$ training images, cross-entropy, minibatch $32$), using the same sweep protocol (five seeds) together with a direct measurement of the noise structure. The noise is of the vanishing type: the per-microbatch gradient variance falls by two and a half orders of magnitude over training and tracks $\|\grad\|^2$ with log--log slope $0.56$ (Fig.~\ref{fig:mlp}b). A constant additive floor is rejected, since the late-training variance sits a factor of $45$ below the intercept of a global affine fit; the regime is interpolation-like but not exactly multiplicative, for which the slope would be $1$. Remark~\ref{rem:alpha} covers this intermediate case: the measured exponent leaves SGD a residual floor of order $\eta^{2.3}$, invisible at this horizon, while the sign floor is noise-agnostic. The schedule preferences follow. SGD's optimal cooldown is $0$ and its final loss is $2.4\times$ worse at $\cf=0.95$; signSGD is the only method cooldown helps ($7.7\times10^{-2}$ to $6.6\times10^{-2}$, plateau over $\cf\in[0.8,0.95]$), remaining more than an order of magnitude above SGD throughout (a factor of $28$ without cooldown); Adam again prefers $\cf^\star=0$, ending $3.5\times$ worse under strong cooldown. We read Adam's SGD-like behavior as a non-stationary self-annealing channel that our stationary analysis excludes: near interpolation the gradient collapses faster than the $\sqrt{\hat v}$ average forgets, so $\hat m/(\sqrt{\hat v}+\epsilon)$ shrinks with the gradient rather than staying $\Theta(1)$, the crossover the $\epsilon$ experiment isolates on the quadratic. This is measured, not inferred: Adam's per-step update norm, as a fraction of a full sign step $\eta\sqrt d$, falls from $0.33$ to $0.02$ over training while the gradient norm drops by $500\times$ (technical appendix), so the update is no longer $\Theta(1)$-normalized and the criterion of Section~\ref{sec:momentum} classifies it correctly. We report training loss because the claim is optimization-level, and in single-epoch pretraining, where cooldown originates, it is also the quantity of record; test accuracy is schedule-insensitive here ($0.98$--$0.99$).

\begin{table}[t]
\centering
\footnotesize
\caption{WSD sweeps: optimal cooldown fraction $\cf^\star$, with the cooldown gain $\E f(x_T)|_{\cf=0}\,/\,\min_\cf\E f(x_T)$ in parentheses. Multiplicative noise unless noted. $^\dagger$Loss at numerical zero ($\lesssim10^{-32}$) for every $\cf$; cooldown offers no benefit.}
\label{tab:exp}
\setlength{\tabcolsep}{3pt}
\begin{tabular}{lccc}
\toprule
Model & SGD & signSGD & normalized \\
\midrule
Quadratic & $0.00$ ($1.0$) & $0.95$ ($1.7{\times}10^{6}$) & $0.75$ ($3.8{\times}10^{5}$) \\
Quartic & $0.00^\dagger$ & $0.95$ ($1.5{\times}10^{6}$) & n/a \\
Quadratic, mixed & $0.95$ ($3.5$) & $0.45$ ($1.2$) & $0.95$ ($1.1$) \\
Deep linear & $0.00^\dagger$ & $0.85$ ($5.8{\times}10^{6}$) & n/a \\
\bottomrule
\end{tabular}
\end{table}

\section{Scope, a Diagnostic, and What This Does Not Explain}\label{sec:scope}

\paragraph{Not the interior WSD fraction.}
In every regime studied here, SGD's optimal cooldown is $0$ (multiplicative) or maximal (additive/mixed), and the normalized methods' optima sit on broad plateaus near full decay. Nothing produces the interior cooldown fractions of roughly $0.2$--$0.35$ that work well in real pretraining runs \citep{haegele2024,schaipp2025,belloni2026}. This is informative: stationary landscape-plus-noise models, ours included, appear structurally unable to produce an interior optimum, pointing to non-stationarity, feature learning, or finite-data effects as the source.

\paragraph{A regime diagnostic.}
The dissociation is actionable only near the multiplicative regime. The governing ratio is $\rho(x)=\sigma_0^2/(\sigma_1^2\|\grad(x)\|^2)$: for $\rho\ll1$ the noise is effectively multiplicative, and SGD-type methods need no cooldown while normalized methods do; for $\rho\gtrsim1$ an additive floor dominates and every method benefits. Both $\sigma_0^2$ and $\sigma_1^2$ are estimable online by regressing per-microbatch gradient variance on $\|\hat g\|^2$, as we do in Section~\ref{sec:exp}. Because $(\sigma_0^2,\sigma_1^2)$ drifts over a long trajectory, the regression should be applied locally in time; the operative question is whether the noise variance vanishes with the gradient. Since $\|\grad\|$ shrinks over training, $\rho$ grows late in a run whenever a genuine additive component exists, so cooldown should matter most toward the end, which is where WSD puts it.

\paragraph{What is and is not new.}
The self-annealing half is classical interpolation theory \citep{schmidt2013,vaswani2019}; the contribution is the contrast (Lemma~\ref{lem:sign}, Props.~\ref{prop:sign} and~\ref{prop:normsgd}): convex last-iterate theory explains why cooldown helps (sub)gradient methods under additive-type noise \citep{schaipp2025} but does not separate optimizers at fixed noise. A tight last-iterate analysis of normalized methods under $(L_0,L_1)$-smoothness, turning the dissociation into an optimal cooldown \emph{shape}, remains open.

\section{Related Work}\label{sec:related}

SDE approximations of stochastic optimizers go back to \citet{li2017}, with adaptive-method scaling rules in \citet{malladi2022}. \citet{compagnoni2026} analyze distributed, compressed, and sign SGD under $(L_0,L_1)$-smoothness with the affine variance model we adopt, giving randomly-sampled-iterate guarantees complementary to our last-iterate question. \citet{shulgin2026} derive hyperparameter and schedule scaling laws for modern optimizers from linear-minimization-oracle bounds, and \citet{liwu2026} obtain optimal schedule shapes under functional scaling laws, flagging adaptivity and momentum as open. Empirically, \citet{haegele2024} establish constant-rate-plus-cooldown as a reliable alternative to cosine schedules, \citet{wen2025} view WSD through a river-valley landscape, \citet{belloni2026} document cross-architecture universality, and \citet{schaipp2025} match practical schedules to convex last-iterate bounds. Under additive-type noise, constant-step SGD has an $\eta$-scaled stationary distribution \citep{dieuleveut2020}, the regime of Proposition~\ref{prop:additive}; under multiplicative noise there is no floor to converge to. Sign descent as a surrogate for Adam is supported by \citet{balles2018} and \citet{kunstner2023}, who trace much of Adam's advantage on transformers to sign-like behavior, while schedule-free averaging \citep{defazio2024} removes decay by averaging iterates instead. Interpolation and strong growth are due to \citet{schmidt2013} and \citet{vaswani2019}, signSGD to \citet{bernstein2018}, with non-convergence counterexamples and error feedback in \citet{karimireddy2019} and relaxed-smoothness analyses in \citet{zhang2020} and \citet{crawshaw2022}.

\section{Conclusion}

Whether learning-rate cooldown helps is a joint property of the gradient noise and the optimizer. Under multiplicative noise SGD self-anneals and prefers a constant rate, while sign-based and normalized methods sit on an $\eta^2$ floor that only decay removes; any additive component then restores a floor for every method. The dissociation persists under $(L_0,L_1)$ geometry, in higher dimension, and with momentum and heavy-tailed noise, and we confirm it on a nonconvex network and on real data. What it does not settle, the interior cooldown fraction used at scale, appears to lie beyond stationary landscape-and-noise geometry.

\onecolumn
\section*{Technical Appendix}
Proposition, lemma, remark, equation, table, and figure numbers below refer to the numbered statements in the main text above.

\section*{A. Exact Stationary Law for signSGD (Proposition 2(ii))}\label{app:lattice}

\paragraph{Setup.}
Consider one coordinate with curvature $\lambda>0$, pure multiplicative scale-family noise with symmetric $Z$, constant step $\eta$, and initialization $x_0=(k+\varphi)\eta$ with offset $\varphi\in(0,1)$. The signSGD update is $x_{t+1}=x_t-\eta\,\sign(\hat g_t)$ with $\hat g_t=\lambda x_t(1+\sigma_1Z_t)$, so
\[
x_{t+1}=x_t-\eta\,\sign(x_t)\,B_t,
\]
where $B_t:=\sign(1+\sigma_1Z_t)\in\{\pm1\}$, and by Lemma~1 the variables $B_t$ are i.i.d.\ with $\Prob(B_t=+1)=p$, independent of the iterate. (For symmetric $Z$ the sign reliability is the same on both sides of the optimum, which is the only place symmetry is used.) Every iterate stays on the lattice $x_0+\eta\mathbb Z$, which does not contain $0$; hence $\hat g_t\neq0$ almost surely, every step has magnitude exactly $\eta$, and the absorbing point $0$ is never reached.

\paragraph{The absolute-value chain.}
Write the positive lattice points as
\[
A_j:=(j+\varphi)\eta,\qquad B_j:=(j+1-\varphi)\eta,\qquad j\ge0 ,
\]
where the $A$-ladder collects the values of $|x|$ realized while $x>0$ and the $B$-ladder those realized while $x<0$ (if $x_0>0$; the labeling is symmetric otherwise). A correct step ($B_t=+1$) moves $|x|$ inward, an incorrect one outward:
\begin{gather*}
A_j\to A_{j-1}\ (j\ge1)\quad\text{and}\quad A_0\to B_0
\quad\text{w.p.\ }p,\\
A_j\to A_{j+1}\quad\text{w.p.\ }q:=1-p,
\end{gather*}
and identically on the $B$-ladder. The only communication between the ladders is the bottom edge $A_0\leftrightarrow B_0$: stepping inward from $A_0=\varphi\eta$ crosses the origin and lands at $-(1-\varphi)\eta$, i.e.\ at $B_0$.

\paragraph{Invariant law by detailed balance.}
The transition graph is a tree (two half-lines joined by one edge), and a Markov chain on a tree is reversible, so its invariant law $\pi$ satisfies detailed balance edge by edge:
\begin{gather*}
\pi(A_j)\,q=\pi(A_{j+1})\,p,\qquad
\pi(B_j)\,q=\pi(B_{j+1})\,p,\\
\pi(A_0)\,p=\pi(B_0)\,p .
\end{gather*}
Hence $\pi(A_j)=\pi(B_j)=\kappa\beta^j$ with $\beta:=q/p<1$ for $p>\tfrac12$, and normalization $2\kappa\sum_{j\ge0}\beta^j=1$ gives $\kappa=(1-\beta)/2$.

\paragraph{Second moment.}
Under $\pi$,
\[
\frac{\E_\pi[x^2]}{\eta^2}
=\frac{1-\beta}{2}\sum_{j\ge0}\beta^j\Big[(j+\varphi)^2+(j+1-\varphi)^2\Big].
\]
Expanding $(j+\varphi)^2+(j+1-\varphi)^2=2(j^2+j)+\varphi^2+(1-\varphi)^2$ and using the geometric sums
$\sum_j\beta^j=\tfrac1{1-\beta}$, $\sum_j j\beta^j=\tfrac{\beta}{(1-\beta)^2}$, $\sum_j j^2\beta^j=\tfrac{\beta(1+\beta)}{(1-\beta)^3}$,
\[
\frac{\E_\pi[x^2]}{\eta^2}
=\frac{\varphi^2+(1-\varphi)^2}{2}
+(1-\beta)\Big[\frac{\beta(1+\beta)}{(1-\beta)^3}+\frac{\beta}{(1-\beta)^2}\Big].
\]
Substituting $\beta=q/p$ and $1-\beta=(2p-1)/p$ simplifies the bracket to $q/(2p-1)+q/(2p-1)^2$, which is Equation~(4):
\[
c(\varphi,p)=\frac{\varphi^2+(1-\varphi)^2}{2}+\frac{q}{2p-1}+\frac{q}{(2p-1)^2}.
\]
Since $\varphi^2+(1-\varphi)^2\ge\tfrac12$, we get $c(\varphi,p)\ge\tfrac14$ for all $\varphi,p$, and averaging over $\varphi\sim\mathrm{Unif}[0,1)$ (the offset induced by a continuous initialization) replaces the first term by $\tfrac13$, giving $\bar c(p)$ of Remark~3. At $p=1$ the walk deterministically oscillates between the two bottom states and $c=\big(\varphi^2+(1-\varphi)^2\big)/2$ exactly, consistent with the formula at $q=0$.

\paragraph{Periodicity and why Proposition 2(i) is stated for pairs.}
The two-ladder graph is bipartite, so the chain has period $2$: the law of $x_T$ oscillates between an even and an odd profile, and $\E[x_T^2]$ need not converge, although time averages converge to $\E_\pi[x^2]$ and each parity profile has the same geometric tails. This is why the granularity bound in Proposition~2(i) is stated for consecutive pairs, $\E[f(x_T)]+\E[f(x_{T+1})]\ge\tfrac18\lambda\eta^2$, rather than as a $\liminf$ of the single-iterate loss: for $p$ close to $1$ and offset $\varphi$ close to $0$ or $1$, the single-iterate expectation genuinely dips to $\tfrac12\lambda\min(\varphi,1-\varphi)^2\eta^2$ on one parity class. All floor statements in this paper use the pair form for this reason.

\section*{B. Drift--Diffusion Comparison (Remark 3)}\label{app:dd}

A drift--diffusion (small-step) approximation replaces the walk by the SDE $dx=-\mu\,\sign(x)\,dt+\sigma\,dW$ with per-step drift $\mu=(2p-1)\eta$ and per-step variance $\sigma^2=\E[(\eta B)^2]-\mu^2=4pq\,\eta^2$. Its stationary density is Laplace, $\propto e^{-2\mu|x|/\sigma^2}$, with second moment $\sigma^4/(2\mu^2)$, i.e.
\[
c_{\mathrm{dd}}(p)=\frac{8p^2q^2}{(2p-1)^2}.
\]
Because the step size $\eta$ coincides with the stationary width $\Theta(\eta)$, there is no scale separation and the approximation degrades as $p\to1$: $c_{\mathrm{dd}}\to0$ while the exact $\bar c(p)\to\tfrac13$. Numerically, the ratio $\bar c/c_{\mathrm{dd}}$ is $1.07$, $1.34$, $2.15$, $6.07$, $20.2$ at $p=0.6,0.7,0.8,0.9,0.95$, and $88.0$ at $p=\Phi(2)\approx0.9772$ (the Gaussian case $\sigma_1=0.5$ used throughout this paper).

\section*{C. Expansion Details for Proposition 6(ii)}\label{app:drift}

With $\theta_t=\angle(\hat g_t,x_t)$ and step $-\eta\hat g_t/\|\hat g_t\|$, the law of cosines gives
\[
R_{t+1}=R_t\sqrt{1-2u\cos\theta_t+u^2},\qquad u:=\eta/R_t .
\]
For $u<1$, $\sqrt{1-2u\cos\theta+u^2}=1-u\cos\theta+\tfrac12u^2\sin^2\theta+O(u^3)$ uniformly in $\theta$, so
\[
R_{t+1}-R_t=-\eta\cos\theta_t+\frac{\eta^2\sin^2\theta_t}{2R_t}+O\!\Big(\frac{\eta^3}{R_t^2}\Big).
\]
Taking conditional expectations and applying Lemma~2 (the joint law of the relevant unit vectors depends on $x_t$ only through $\hat x_t$) yields the drift display with radius-independent coefficients $\bar c(\hat x_t)$ and $\bar s(\hat x_t)$. When $c_-=\inf_{\hat x}\bar c(\hat x)>0$, the drift is bounded above by $-\eta c_-+\eta^2/(2R_t)$, which is negative for $R_t>\eta/(2c_-)$; the radius is therefore driven into a band of width $\Theta(\eta)$ and, by the granularity bound of Proposition~6(i), cannot collapse below scale $\eta$ either. The main text presents part~(ii) as a leading-order equilibrium statement; the fully rigorous floor is part~(i), which needs no expansion.

\section*{D. Complete Simulation Configurations}\label{app:config}

Common defaults: curvature $\lambda=1$, noise scale $\sigma_1=0.5$, standard Gaussian $Z$, NumPy \texttt{default\_rng} with fixed seeds recorded in the code. Stationary quantities are time averages after burn-in. All scripts run on CPU in minutes.

\paragraph{One-step SGD ratio (Prop.~1).}
$\E[(1-\eta(1+\sigma_1Z))^2]$ estimated from $2\times10^7$ draws at $\eta\in\{0.70,0.75,0.80,0.85,0.90\}$; maximum absolute deviation from $r(\eta)$ was $6\times10^{-5}$. Long-horizon decay checked at $\eta^\star=0.8$ until floating-point underflow.

\paragraph{signSGD floors (Prop.~2, Table~2).}
Direct Bernoulli-walk simulation for Table~2: $\eta=0.1$, $x_0\sim\mathrm{Unif}(0.5,1.5)$, $400$ chains, $4\times10^5$ steps, burn-in $5\times10^4$. Gaussian-noise floors: $\eta\in\{0.02,0.04,0.08,0.16\}$, $\sigma_1\in\{0.5,1,2\}$, $300$ chains, $3\times10^5$ steps, burn-in $6\times10^4$; exponents from a log-log fit. The absorption artifact was reproduced by setting $x_0=0.3$ with $\eta=0.003$ (so $x_0\in\eta\mathbb Z$): the iterate can then reach the lattice point $0$, where the multiplicative gradient and hence the update vanish.

\paragraph{Sign reliability (Lemma~1).}
$2\times10^6$ draws per law and scale; laws: standard Gaussian, Student-$t_2$ (raw scale), lognormal standardized to zero mean and unit variance, evaluated at $x\in\{1,0.1,0.01\}$.

\paragraph{$(L_0,L_1)$ quartic (Props.~4, 5).}
$f(x)=x^2/2+x^4/4$. In-basin: $x_0=0.3$, $\eta=0.3$, $10^5$ chains; the reported conditional contraction factor is the sample mean of $r(s(x_t),\eta)$. Out-of-basin: $\eta=0.8$, $x_0=1.5$. WSD sweep: $T=2000$, $x_0\sim\mathrm{Unif}(0.25,0.35)$, $4000$ chains $\times\,3$ seeds, peak-rate grids $\{0.1,0.2,0.3,0.5,0.7\}$ (SGD) and $\{0.003,0.01,0.03,0.1\}$ (signSGD).

\paragraph{Additive floor (Prop.~3).}
$\sigma_0=0.2$, $\eta=0.05$, $2000$ chains, $2\times10^5$ steps, burn-in $5\times10^4$.

\paragraph{Normalized SGD in $d=20$ (Prop.~6).}
Spectrum $\lambda_i$ log-spaced from $1$ to $10^{-3}$; direction $x\propto\mathbf 1$; alignments from $2\times10^6$ draws at radii $1,0.1,0.01$; the isotropic comparison uses $\lambda_i\equiv1$. Floor exponent from $\eta\in\{0.02,0.04,0.08\}$, $100$ chains, $1.5\times10^5$ steps, burn-in $3\times10^4$.

\paragraph{Momentum (Sec.~8).}
$\beta=0.9$, $\eta=0.02$. Heavy-ball in the averaged form $m_t=\beta m_{t-1}+(1-\beta)\hat g_t$, $2\times10^4$ chains, $2.5\times10^4$ steps, tracked in log space with periodic rescaling of the joint state $(x_t,m_t)$, which is exact because the dynamics are linear in the state. Sign-with-momentum: step $-\eta\,\sign(m_t)$, $200$ chains, $2\times10^5$ steps, burn-in $5\times10^4$.

\paragraph{AdamW (Sec.~10).}
$\beta_1=0.9$, $\beta_2=0.999$, bias correction on, $\eta\in\{0.005,0.01,0.02,0.04\}$, $200$ chains, $2\times10^5$ steps, burn-in $5\times10^4$; configurations $(\epsilon,\text{wd})\in\{(10^{-12},0),(10^{-12},0.01),(10^{-4},0)\}$ with decoupled weight decay.

\paragraph{Heavy tails (Sec.~9).}
signSGD under standard Cauchy $Z$: same protocol as the Gaussian floors. SGD under Student-$t_{1.5}$: $\eta=0.8$, $2\times10^5$ runs of $300$ steps, evolved as $\log|x_t|$; the reported mean of $x^2$ is the log-sum-exp sample mean.

\paragraph{WSD sweeps (Table~3, Fig.~1).}
Schedule: $\eta_t=\eta$ for $t<(1-\cf)T$, then linear decay to $0$; $\cf\in\{0,0.05,\dots,0.95\}$; the reported loss at each $\cf$ is the minimum over the peak-rate grid of the mean last-iterate loss. For SGD under multiplicative noise this min-over-grid loss increases monotonically in $\cf$ across the entire sweep (Fig.~1a), so $\cf^\star=0$ is not an artifact of grid resolution.
Quadratic: $d=20$, spectrum as above, $x_0=\mathbf 1/\sqrt d$, $T=2000$, $512$ chains per configuration; grids $\{0.1,0.2,0.4,0.8,1.2,1.5\}$ (SGD) and $\{0.003,0.01,0.03,0.1,0.3\}$ (signSGD, normalized SGD); mixed regime adds $\sigma_0=0.2$.
Deep linear network: $y=W_2W_1x$ with $W_1,W_2\in\mathbb R^{10\times10}$, planted orthogonal teacher $W^\star$ (QR factor of a Gaussian matrix), population loss $\tfrac12\|W_2W_1-W^\star\|_F^2$, exact population gradients with coordinatewise multiplicative noise, initialization $I+0.1\,\mathcal N(0,1)^{10\times10}$, $T=4000$, $8$ runs per configuration; grids $\{0.02,0.05,0.1,0.2,0.4\}$ (SGD) and $\{0.001,0.003,0.01,0.03\}$ (signSGD).

\section*{E. Real-Data Experiment (Fig.~2)}\label{app:mlp}

\paragraph{Task and model.}
Scikit-learn \texttt{digits} ($8\times8$ grayscale images, features scaled to $[0,1]$), shuffled with a fixed seed and split $1{,}500$ train / $297$ test. Two-layer ReLU network $64\to128\to10$ (He initialization, biases zero; $9{,}610$ parameters), softmax cross-entropy, minibatches of size $32$ drawn i.i.d.\ with replacement.

\paragraph{WSD sweep.}
$T=4{,}000$ steps; cooldown fractions $\cf\in\{0,0.1,\dots,0.9,0.95\}$; peak-rate grids $\{0.05,0.1,0.2,0.4\}$ (SGD), $\{10^{-4},3\times10^{-4},10^{-3},3\times10^{-3},10^{-2}\}$ (signSGD), $\{3\times10^{-4},10^{-3},3\times10^{-3},10^{-2}\}$ (Adam, $\beta_1=0.9$, $\beta_2=0.999$, $\epsilon=10^{-8}$, bias correction); five seeds controlling both initialization and minibatch order. Reported: mean over seeds of the final full-batch training loss, minimized over the peak-rate grid per $(\text{optimizer},\cf)$; bands in Fig.~2a are min--max over seeds at the selected peak rate. Final test accuracies at the per-optimizer best configuration: $0.991$ (SGD), $0.986$ (signSGD), $0.994$ (Adam); accuracy varies by less than $0.01$ across $\cf$ for every optimizer.

\paragraph{Noise measurement.}
Along a reference run (SGD, constant $\eta=0.1$), at checkpoints $t\in\{0,50,100,200,400,800,1600,3200,6400,12800\}$ we compute the full-batch gradient $G$ and the empirical variance $V=\frac{1}{64}\sum_{b=1}^{64}\|g_b-G\|^2$ over $64$ fresh microbatches of size $32$ (gradients flattened over all parameters). Measured values ($t$, training loss, $\|G\|^2$, $V$): $(0,\,2.46,\,9.0\times10^{-1},\,1.31)$; $(50,\,0.94,\,2.9\times10^{-1},\,0.70)$; $(100,\,0.50,\,1.1\times10^{-1},\,0.50)$; $(200,\,0.28,\,8.2\times10^{-2},\,0.40)$; $(400,\,0.17,\,5.0\times10^{-2},\,0.32)$; $(800,\,0.10,\,1.1\times10^{-2},\,0.18)$; $(1600,\,0.056,\,3.6\times10^{-3},\,0.12)$; $(3200,\,0.028,\,4.3\times10^{-3},\,6.8\times10^{-2})$; $(6400,\,0.011,\,2.7\times10^{-4},\,2.1\times10^{-2})$; $(12800,\,4.3\times10^{-3},\,3.9\times10^{-5},\,3.7\times10^{-3})$. The log--log regression of $V$ on $\|G\|^2$ has slope $0.56$. A global affine fit $V=a+b\|G\|^2$ gives $a=0.166$, $b=1.36$; the final measured $V$ is a factor of $45$ below that intercept, which is how a constant additive floor is rejected: the affine model with trajectory-constant $(\sigma_0^2,\sigma_1^2)$ is only locally valid, and the operative fact is that the noise variance collapses together with the gradient.

\paragraph{Adam update-magnitude measurement.}
To test the explanation offered in the main text for Adam's SGD-like schedule preference, we track the per-step update norm $\|\theta_{t+1}-\theta_t\|$ as a fraction of a full sign step $\eta\sqrt d$ ($d=9{,}610$ parameters), averaged over a $10$-step window, along an Adam run at constant $\eta=3\times10^{-3}$ (the peak rate selected by the sweep at $\cf=0$), together with the full-batch gradient norm. Measured values ($t$, fraction, $\|\grad\|$): $(50,\,0.33,\,0.58)$; $(100,\,0.22,\,0.28)$; $(200,\,0.14,\,0.15)$; $(400,\,0.12,\,0.18)$; $(800,\,0.12,\,0.087)$; $(1600,\,0.074,\,0.038)$; $(3200,\,0.032,\,0.0050)$; $(4000,\,0.022,\,0.0012)$. At $\eta=10^{-3}$ the fraction falls from $0.43$ to $0.07$ over the same horizon. A $\Theta(1)$-normalized method would hold this fraction roughly constant; instead it declines by an order of magnitude as the gradient collapses, confirming that on this task the $\sqrt{\hat v}$ memory converts Adam into an effectively gradient-proportional, self-annealing update.

\section*{F. Code}\label{app:code}

The \texttt{code/} directory is a self-contained package (NumPy/SciPy/Matplotlib/scikit-learn) with a one-command entry point, \texttt{run\_all.sh}, that regenerates every result and figure into \texttt{results/}. Shared utilities live in \texttt{common.py} (deterministic seeding and the WSD schedule) and \texttt{mlp.py} (the two-layer network); all randomness is seeded deterministically, so every number reproduces on rerun. The experiment scripts are \texttt{verify\_claims.py} (all stationary checks, exponents, and constants), \texttt{wsd\_sweeps.py} (quadratic and deep-linear WSD sweeps), \texttt{quartic\_wsd.py} (quartic sweep and the one-step ratio check), \texttt{mlp\_realdata.py} (real-data experiment of Fig.~2), and \texttt{adam\_step\_norm.py} (the Adam update-magnitude measurement above); \texttt{make\_figure1.py} and \texttt{make\_figure2.py} render the figures from the saved outputs. Total runtime is under forty minutes on a laptop CPU.

\bibliography{references}

@misc{belloni2026,
  author    = {Belloni, A. and Noci, L. and Orvieto, A.},
  title     = {Universal Dynamics of Warmup Stable Decay: Understanding {WSD} beyond Transformers},
  year      = {2026},
  note      = {ICML 2025 Workshops on High-dimensional Learning Dynamics (HiLD) and Methods and Opportunities at Small Scale (MOSS); arXiv:2601.09000}
}

@inproceedings{bernstein2018,
  author    = {Bernstein, J. and Wang, Y.-X. and Azizzadenesheli, K. and Anandkumar, A.},
  title     = {{signSGD}: Compressed Optimisation for Non-Convex Problems},
  booktitle = {Proceedings of the 35th International Conference on Machine Learning (ICML)},
  year      = {2018}
}

@inproceedings{compagnoni2026,
  author    = {Monzio Compagnoni, E. and Islamov, R. and Proske, F. N. and Lucchi, A. and Orvieto, A. and Gorbunov, E.},
  title     = {On the Interaction of Batch Noise, Adaptivity, and Compression, under $(L_0,L_1)$-Smoothness: An {SDE} Approach},
  booktitle = {Proceedings of the 43rd International Conference on Machine Learning (ICML)},
  year      = {2026},
  note      = {arXiv:2506.00181}
}

@inproceedings{balles2018,
  author    = {Balles, L. and Hennig, P.},
  title     = {Dissecting {Adam}: The Sign, Magnitude and Variance of Stochastic Gradients},
  booktitle = {Proceedings of the 35th International Conference on Machine Learning (ICML)},
  year      = {2018}
}

@inproceedings{crawshaw2022,
  author    = {Crawshaw, M. and Liu, M. and Orabona, F. and Zhang, W. and Zhuang, Z.},
  title     = {Robustness to Unbounded Smoothness of Generalized {SignSGD}},
  booktitle = {Advances in Neural Information Processing Systems (NeurIPS)},
  year      = {2022}
}

@article{dieuleveut2020,
  author    = {Dieuleveut, A. and Durmus, A. and Bach, F.},
  title     = {Bridging the Gap between Constant Step Size Stochastic Gradient Descent and {Markov} Chains},
  journal   = {Annals of Statistics},
  volume    = {48},
  number    = {3},
  year      = {2020}
}

@inproceedings{defazio2024,
  author    = {Defazio, A. and Yang, X. A. and Mehta, H. and Mishchenko, K. and Khaled, A. and Cutkosky, A.},
  title     = {The Road Less Scheduled},
  booktitle = {Advances in Neural Information Processing Systems (NeurIPS)},
  year      = {2024}
}

@article{hu2024minicpm,
  author    = {Hu, S. and Tu, Y. and Han, X. and He, C. and Cui, G. and others},
  title     = {{MiniCPM}: Unveiling the Potential of Small Language Models with Scalable Training Strategies},
  journal   = {arXiv preprint arXiv:2404.06395},
  year      = {2024}
}

@inproceedings{shamir2013,
  author    = {Shamir, O. and Zhang, T.},
  title     = {Stochastic Gradient Descent for Non-smooth Optimization: Convergence Results and Optimal Averaging Schemes},
  booktitle = {Proceedings of the 30th International Conference on Machine Learning (ICML)},
  year      = {2013}
}

@inproceedings{haegele2024,
  author    = {H{\"a}gele, A. and Bakouch, E. and Kosson, A. and Ben Allal, L. and von Werra, L. and Jaggi, M.},
  title     = {Scaling Laws and Compute-Optimal Training Beyond Fixed Training Durations},
  booktitle = {Advances in Neural Information Processing Systems (NeurIPS)},
  year      = {2024}
}

@inproceedings{karimireddy2019,
  author    = {Karimireddy, S. P. and Rebjock, Q. and Stich, S. U. and Jaggi, M.},
  title     = {Error Feedback Fixes {SignSGD} and other Gradient Compression Schemes},
  booktitle = {Proceedings of the 36th International Conference on Machine Learning (ICML)},
  year      = {2019}
}

@inproceedings{kunstner2023,
  author    = {Kunstner, F. and Chen, J. and Lavington, J. W. and Schmidt, M.},
  title     = {Noise Is Not the Main Factor Behind the Gap Between {SGD} and {Adam} on Transformers, but Sign Descent Might Be},
  booktitle = {International Conference on Learning Representations (ICLR)},
  year      = {2023}
}

@inproceedings{li2017,
  author    = {Li, Q. and Tai, C. and E, W.},
  title     = {Stochastic Modified Equations and Adaptive Stochastic Gradient Algorithms},
  booktitle = {Proceedings of the 34th International Conference on Machine Learning (ICML)},
  year      = {2017}
}

@inproceedings{liwu2026,
  author    = {Li, B. and Wang, Z. and Chen, F. and Zhao, S. and Zheng, R. and Wu, L.},
  title     = {Optimal Learning-Rate Schedules under Functional Scaling Laws: Power Decay and Warmup--Stable--Decay},
  booktitle = {Proceedings of the 39th Conference on Learning Theory (COLT)},
  year      = {2026},
  note      = {arXiv:2602.06797}
}

@inproceedings{malladi2022,
  author    = {Malladi, S. and Lyu, K. and Panigrahi, A. and Arora, S.},
  title     = {On the {SDEs} and Scaling Rules for Adaptive Gradient Algorithms},
  booktitle = {Advances in Neural Information Processing Systems (NeurIPS)},
  year      = {2022}
}

@article{schaipp2025,
  author    = {Schaipp, F. and H{\"a}gele, A. and Taylor, A. and Simsekli, U. and Bach, F.},
  title     = {The Surprising Agreement Between Convex Optimization Theory and Learning-Rate Scheduling for Large Model Training},
  journal   = {arXiv preprint arXiv:2501.18965},
  year      = {2025}
}

@article{schmidt2013,
  author    = {Schmidt, M. and Le Roux, N.},
  title     = {Fast Convergence of Stochastic Gradient Descent under a Strong Growth Condition},
  journal   = {arXiv preprint arXiv:1308.6370},
  year      = {2013}
}

@article{shulgin2026,
  author    = {Shulgin, E. and von R{\"u}tte, D. and Zhang, T. H. and Ajroldi, N. and Sch{\"o}lkopf, B. and Orvieto, A.},
  title     = {Deriving Hyperparameter Scaling Laws via Modern Optimization Theory},
  journal   = {arXiv preprint arXiv:2603.15958},
  year      = {2026}
}

@inproceedings{vaswani2019,
  author    = {Vaswani, S. and Bach, F. and Schmidt, M.},
  title     = {Fast and Faster Convergence of {SGD} for Over-Parameterized Models and an Accelerated Perceptron},
  booktitle = {Proceedings of the 22nd International Conference on Artificial Intelligence and Statistics (AISTATS)},
  year      = {2019}
}

@inproceedings{wen2025,
  author    = {Wen, K. and Li, Z. and Wang, J. and Hall, D. and Liang, P. and Ma, T.},
  title     = {Understanding Warmup--Stable--Decay Learning Rates: A River Valley Loss Landscape View},
  booktitle = {International Conference on Learning Representations (ICLR)},
  year      = {2025}
}

@inproceedings{zhang2020,
  author    = {Zhang, J. and He, T. and Sra, S. and Jadbabaie, A.},
  title     = {Why Gradient Clipping Accelerates Training: A Theoretical Justification for Adaptivity},
  booktitle = {International Conference on Learning Representations (ICLR)},
  year      = {2020}
}

\end{document}